\UseRawInputEncoding
\documentclass[letterpaper, 10 pt, conference]{ieeeconf}
\IEEEoverridecommandlockouts                             
\usepackage{graphicx} 
\usepackage{cite}
\usepackage{hyperref}
\usepackage{upgreek}
\usepackage{float}
\usepackage{diagbox}
\makeatletter
\let\NAT@parse\undefined
\makeatother
\hypersetup{
colorlinks=true,
linkcolor=cyan,
filecolor=blue,      
urlcolor=red,
citecolor=green,
}
\usepackage{subfigure}
\usepackage{epsfig} 
\usepackage{amsmath} 
\usepackage{float}
\usepackage{cite}
\usepackage{accents}
\usepackage{hyperref}
\usepackage{algorithm}
\usepackage{algorithmic}
\usepackage{xcolor}
\usepackage{bm}
\usepackage{url}
\usepackage{mathrsfs}
\usepackage{multirow}
\usepackage{amsthm}
\usepackage{amsmath, amssymb, amscd, amsthm, amsfonts}
\newtheorem{theorem}{Theorem}

\newtheorem{assumption}{Assumption}

\newtheorem{remark}{Remark}
\newtheorem{proposition}{Proposition}
\newtheorem{definition}{Definition}

\title{\LARGE \bf
Wasserstein Distributionally Robust Chance Constrained Trajectory Optimization for Mobile Robots within Uncertain Safe Corridor}
\author{Shaohang Xu$^{1,2}$,  Haolin Ruan$^2$, Wentao Zhang$^1$, Yian Wang$^{1}$,
Lijun Zhu$^{1}$ and Chin Pang Ho$^{2}$
\thanks{$^{1}$ School of Artificial Intelligence and Automation, Huazhong University of Science and Technology, China, shaohangxu@hust.edu.cn, wentaozhang@hust.edu.cn, yianwang@hust.edu.cn, ljzhu@hust.edu.cn}%
\thanks{$^{2}$ School of Data Science, City University of Hong Kong, HKSAR, shaohanxu2-c@my.cityu.edu.hk, haolin.ruan@my.cityu.edu.hk, clint.ho@cityu.edu.hk}%
}

\begin{document}

\maketitle
\maketitle
\thispagestyle{empty}
\pagestyle{empty}
\begin{abstract}
Safe corridor-based Trajectory Optimization (TO) presents an appealing approach for collision-free path planning of autonomous robots, offering global optimality through its convex formulation. The safe corridor is constructed based on the perceived map, however, the non-ideal perception induces uncertainty, which is rarely considered in trajectory generation. In this paper, we propose Distributionally Robust Safe Corridor Constraints (DRSCCs) to consider the uncertainty of the safe corridor. Then, we integrate DRSCCs into the trajectory optimization framework using Bernstein basis polynomials.  Theoretically, we rigorously prove that the trajectory optimization problem incorporating DRSCCs is equivalent to a computationally efficient, convex quadratic program.  Compared to the nominal TO, our method enhances navigation safety by significantly reducing the infeasible motions in presence of uncertainty. Moreover, the proposed approach is validated through two robotic applications, a micro Unmanned Aerial Vehicle (UAV) and a quadruped robot Unitree A1. 
\end{abstract}

\section{Introduction}
We consider trajectory optimization for mobile robots navigating from an initial state to a goal state while avoiding obstacles. It is known that the collision-free constraints are in general nonlinear and non-convex, and the optimization solvers cannot guarantee to obtain the global optimum with these non-convex constraints.  In contrast, \emph{safe corridor}, defined as the collection of convex decomposition of the collision-free space, is proposed to address this issue  \cite{gao2018online,deits2015efficient,chen2016online,liu2017planning,sun2021fast}.
The optimization problem with the safe corridor constraints is usually a convex program, which can be efficiently computed to obtain the global optimum. However, since perception errors are inevitable in real-world applications,   the safe corridor obtained from the perceived obstacle map is not exactly known. Also, the optimization problems might be sensitive to parameter specifications \cite{ben2009robust}, and thus it may cause catastrophic navigation failures on the real robots when the uncertainty of the safe corridor is ignored.

In this paper, we aim to generate collision-free trajectories considering the uncertainty of the safe corridor. To this end, we propose a method called Distributionally Robust Safe Corridor Constrained Trajectory Optimization (DRSCC-TO) to ensure the safety of navigation for mobile robots in uncertain environments. The motivation for our approach is illustrated in Fig. \ref{fig:example}. In the left figure, the nominal TO approach, which does not consider uncertainty, aims to minimize energy consumption and generates the blue trajectory. However, this blue trajectory becomes infeasible in the right figure, where a small and unknown region is occupied by obstacles that are not well perceived. In contrast, the red trajectory, resulting from our robust trajectory optimization approach, remains collision-free in both cases. There are two highlights of our method:
 
\begin{figure}
\centering

\includegraphics[width=8cm]{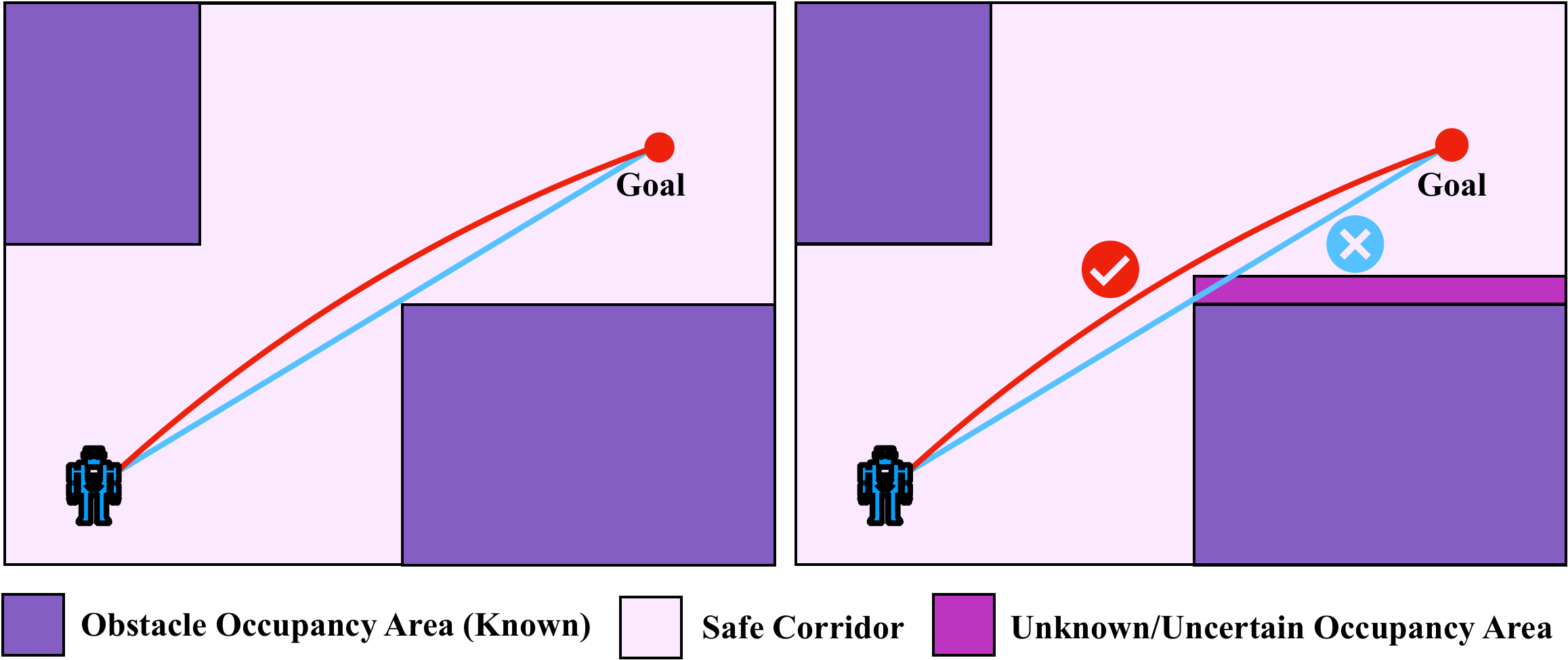}\label{fig:example}

\caption{An example of TO with (red) and without (blue) considering uncertainty. When the safe corridor is exactly known (left), both methods are collision-free, and the blue trajectory is shorter and could minimize the control input. However, when the obstacle information is not perfect and/or there is an unknown obstacle region (right), the blue trajectory will be infeasible, leading to unacceptable collisions. }
\label{fig:quad1}
\end{figure}

\begin{itemize}
	\item Distributionally robust chance constraints: We leverage the Wasserstein metric to model the uncertainty of the safe corridor. To the best of our knowledge, this is the first time the distributional ambiguity of the safe corridor has been considered in the literature.
	With prescribed confidence, we can guarantee there are no infeasible motions and thus improve navigation safety. 
	\item Convexity: In general, collision-free trajectory optimization problems are non-convex. In this paper, we rigorously prove that the optimization problem in the proposed method is equivalent to a convex Quadratic Program (QP), which is known to be practically viable for real robots. In contrast, distributionally robust (chance-constrained) optimization problems are often computationally challenging due to their non-convexity \cite{hakobyan2021wasserstein}, and thus are usually limited to simulation without real robot applications. 
\end{itemize} 

\subsection{Related Work}

\textbf{Safe Corridor based TO} 
has been well studied for collision-free navigation of mobile robots. \cite{deits2015efficient}  proposes a semidefinite program-based approach for collision-free UAV navigation, which leverages a greedy finding algorithm for large convex regions of the collision-free space \cite{deits2015computing}.
The efficient operations of the octree-map representation are proposed in \cite{chen2016online} to obtain the safe corridor, and the safe corridor constraints are incorporated into a minimum snap TO framework \cite{mellinger2011minimum}. Similarly, \cite{liu2017planning} proposes a two-step algorithm to construct the polyhedron-like safe corridor. To guarantee the feasibility of the overall trajectory, \cite{gao2018online} further extends \cite{chen2016online} by using Bernstein basis polynomial as the trajectory representation, instead of the original discretization points. The time allocation issue of safe corridor-based TO is addressed in \cite{sun2021fast}  through a bilevel optimization approach.  In all the above methods, the safe corridor is considered to be deterministic without uncertainty. 

In this paper, we extend the safe corridor-based TO from the \emph{deterministic} formulation to the \emph{distributionally robust} formulation. In particular, we build on the Bernstein basis method introduced in \cite{gao2018online}, and provide the distributionally robust chance-constrained counterpart to formally consider the distributional ambiguity of the safe corridor. 

\textbf{Robust TO} 
is proposed to deal with the uncertainty in real-robot applications, such as measurement errors, external disturbances, etc. Set-based approaches provide a safety guarantee for robotic systems with bounded (additional) disturbances. A closely related research topic in the Model Predictive Control (MPC) community is tube-based MPC \cite{mayne2005robust}, which tightens the constraints offline, and then online solves an optimization problem as the nominal MPC does. Similar to the idea of tube-based MPC, a robust feedback motion planning method is presented in \cite{majumdar2017funnel}, which computes the outer approximation of the reachable set under uncertainty offline, and then sequentially composes the motion plans online. Chance-constrained optimization \cite{charnes1959chance} provides the basic tools for motion planning of stochastic robot systems. A chance-constrained TO method is proposed in \cite{blackmore2011chance} to consider the Gaussian noise in both the robot and obstacle dynamics. A robust cost function in TO is introduced in \cite{manchester2017dirtrel} by computing the ellipsoid bounds around the nominal trajectory. Recently, distributionally robust optimization (DRO) based approaches have been proposed to consider the uncertainty variable without perfectly known probability distributions in robotic motion planning. \cite{nakka2022trajectory} presents a nonlinear optimal control-based TO method for stochastic dynamics systems with distributionally robust chance constraints, which is computed by sequential quadratic programming. The Wasserstein metric is used in \cite{hakobyan2021wasserstein} to consider uncertain moving obstacles and a distributionally robust MPC is proposed to achieve a probabilistic guarantee of the out-of-sample risk. However, it is only validated in numerical simulation.   

Our method is also a robust TO approach. Different from above-mentioned methods, we leverage DRO techniques to prove our TO problem is a convex QP, making it practically useful for mobile robot applications. In contrast, the above-mentioned methods leverage non-convex optimization problems which are difficult, if not impossible, to obtain the global optima in real-time. As a result, our method can be easily deployed onto real robots, while the above methods are usually limited to simulation. Also, it is worth noting that we mainly focus on the uncertainty of the safe corridor, which is rarely considered in the literature.

\subsection{Contributions}
The contributions of this paper are summarized as follows:	
\begin{itemize}
	\item We propose a robust trajectory optimization framework considering the distributional ambiguity of the safe corridor for mobile robots. 
	\item Theoretically, we prove that the distributionally robust chance constrained trajectory optimization problem is equivalent to a convex QP. 
	\item By conducting extensive numerical tests, we demonstrate that our methodology has the capacity to significantly decrease the number of infeasible motions in comparison to the nominal approach.
	\item We evaluate the proposed method on a quadruped robot and a UAV, showing that our method is practically viable for safety-critical robot navigation.
\end{itemize}
\section{Prelimenary}

In this section, we review deterministic TO using Bernstein basis polynomials \cite{gao2018online}, which could rigorously guarantee that the generated trajectory is within the safe corridors.

\subsection{Safe Corridor Constraint}
The number of the planning dimension is denoted as $m$. 
In general, the safe corridor $\mathbb{S}_c$ is defined as one of the sub-sets of the passable region $\mathbb{S}_p \in \mathbb{R}^m$, and is the union set of a sequence of feasible regions $\mathbb{S}_i$, i.e.,
\begin{equation}
\label{safe_corridor}
\mathbb{S}_c \triangleq \bigcup\limits_{i=1}^{N} \mathbb{S}_i
\end{equation}
where the safe corridor consists of $N$ sub-feasible spaces $\mathbb{S}_i \in \mathbb{S}_p$. In particular, we assume that the sub-feasible spaces are cube-like, and the adjacent spaces have overlapping regions:
$$\mathbb{S}_i \cap \mathbb{S}_{i+1} \neq \emptyset, \forall i \in \{1,\cdots, N-1\},$$
$$
	\mathbb{S}_i \triangleq \{\mathbf{p} \vert \mathbf{p}\in \mathbb{R}^m, \mathbf{s}^l_i \leq \mathbf{p} \leq \mathbf{s}^u_i \},
$$
where $\mathbf{p}$ is the position in the motion space, $\mathbf{s}^l_i \in \mathbb{R}^m$ and $\mathbf{s}^u_i \in \mathbb{R}^m$ are the known lower and upper boundaries, respectively.
In practice, the sub-feasible spaces can be obtained from an initial path by sampling, graph search or other near-optimal methods \cite{gao2016online, gao2019flying,liu2017planning,gao2018online,chen2016online}.
The initial path is piece-wise linear and consists of $N+1$ collision-free path points, denoted as 
$$(\mathbf{p}_0,T_0), \cdots, (\mathbf{p}_i,T_i), \cdots, (\mathbf{p}_N,T_N)$$
where $\mathbf{p}_0$ is the initial position, $T_0$ is the start time,  $\mathbf{p}_i$ is the $i$-th intermediate point and $T_i$ is the corresponding arrival time, $\mathbf{p}_N$ is the end position, i.e., the goal position.
Moreover, the intermediate points are in the corresponding overlapping regions, i.e.,
$$
\mathbf{p}_i \in \left (\mathbb{S}_{i} \cap \mathbb{S}_{i+1} \right ), \ \forall i \in {1,\cdots,N-1}
$$
From the initial path, we define the $i$-th piece of trajectory as $B_i(t): \mathbb{R} \rightarrow \mathbb{R}^m$. Note that  $B_i(t)$ is related to the $i$-th sub-feasible region $\mathbb{S}_i$.
Then, if the following constraint holds for  $B_i(t)$, we can guarantee that $B_i(t)$ is collision-free:
\begin{equation}
\label{eq_safe_corridor_constraint}
B_i(t) \in \mathbb{S}_i,\ \forall t \in [T_{i-1},T_{i}],
\end{equation}

\subsection{Deterministic TO using Bernstein Basis Polynomial}
The trajectory $B_i(t)$ has $m$ dimensions, denoted separately by $B_{i,1}(t),\cdots, B_{i,m}(t)$, respectively.
Let $B_{i,\mu}(t)$ be an $n$-degree Bernstein basis polynomial, i.e., Bezier curve, for each $\mu \in \{1,\cdots,m\}$. Then $B_{i,\mu}(t)$ could be described by a set of control points:
\begin{equation}
\label{bezier_si}
B_{i,\mu}(t)=\sum \limits_{j=0}^n c_{i,\mu}^jb_n^j(t),
\end{equation}
where $c_{i,\mu}^j$ is the $j$-th control point for $B_i(t)$ in the $\mu$-th dimension, and $b_n^j(t)$ is the Bernstein polynomial basis:
\begin{equation*}
b_n^j(t) \triangleq \left ( \begin{array}{c}
n\\
j
\end{array} \right ) \cdot t^j \cdot (1-t)^{(n-j)},
\end{equation*}
where $t\in [0,1]$. The most attractive advantage is the convex hull property of Bezier curves, i.e., a Bezier curve is guaranteed to be within the convex hull of its control points. By leveraging this property, we can prove the following proposition.
\begin{proposition}
Suppose the trajectory $B_i(t)$ is a Bezier curve \eqref{bezier_si}.  Then, the constraint \eqref{eq_safe_corridor_constraint} holds if
\begin{equation}
\label{safe_constraint}
\mathbf{e}_\mu^\top {\mathbf{s}}^l_i \leq c_{i,\mu}^j \leq \mathbf{e}_\mu^\top {\mathbf{s}}^u_i, \ \forall \mu \in \{1,\cdots,m\}, j\in\{0,\cdots,n\},
\end{equation}
where $\mathbf{e}_\mu$ is the standard basis in $\mathbb{R}^m$ where the entry associated with the $\mu$-th dimension is one.
\end{proposition}
 \begin{remark} \label{SCC_remark} (Safe Corridor Constraint, SCC) In this paper, we call the constraints \eqref{safe_constraint} as Safe Corridor Constraints (SCCs). The above proposition shows that we can use the constraints \eqref{safe_constraint} to guarantee the safety of $B_i(t)$, instead of the originally intractable constraint \eqref{eq_safe_corridor_constraint}. Moreover, it is clear that if we take $c_{i,\mu}^j$ as a decision variable, the constraint \eqref{safe_constraint} is a linear constraint and can be incorporated into TO. It is worth noting that, both $s_{i,\mu}^l$ and $s_{i,\mu}^u$ are known exactly here. Therefore, the constraint  \eqref{safe_constraint} is indeed \textbf{deterministic}. In the next section, we will introduce the distributionally robust counterpart of \eqref{safe_constraint}.
\end{remark}
Given the initial path, the trajectory optimization problem for mobile robots within \emph{deterministic} SCCs could be formulated as follows. Please see \cite{gao2018online} for a detailed explanation.
\begin{subequations}
\label{nominal_QP}
\begin{align}
\displaystyle \min
_{\mathbf{c}} 
& \quad \mathbf{c}^\top \mathbf{Q} \mathbf{c}, \label{QP_obj} \\
{\rm{s.t.}}  & \quad \mathbf{e}_\mu^\top {\mathbf{s}}^l_i \leq c_{i,\mu}^j \leq \mathbf{e}_\mu^\top {\mathbf{s}}^u_i,  \label{SCC in TO} \\
& \quad a_{i,\mu}^{0,j}=c_{i,\mu}^{j}, \label{new_a_1}\\
& \quad a_{i,\mu}^{l,j}=\frac{n!}{(n-l)!}\cdot (a_{i,\mu}^{l-1,i+1}-a_{i,\mu}^{l-1,i}), \label{new_a_2}\\
&\quad  a_{1,\mu}^{l,0} \cdot (\tau_i)^{1-l} = d^{(l)}_{1,\mu}, \label{waypoint_cons_0}\\
&\quad  a_{N,\mu}^{l,n} \cdot (\tau_i)^{1-l} = d^{(l)}_{N,\mu}, \label{waypoint_cons_1}\\
&\quad a_{i,\mu}^{\phi,n}\cdot (\tau_i)^{1-\phi}= a_{i+1,\mu}^{\phi,0}\cdot (\tau_{i+1})^{1-\phi}, \label{way_point_cons_2} \\
& \quad a_{\min}^g\leq a_{\mu,i-1}^{g,j}(\tau_i^{1-\phi}) \leq a_{\max}^g, \label{limit_constraints}
\end{align}
\end{subequations}
where  $\mathbf{c} \triangleq [c_{1,x}^0,\cdots,c_{N,x}^n,c_{1,y}^0,\cdots,c_{N,y}^n,c_{1,z}^0,\cdots,c_{N,z}^n]^\top$ is the vector of all control points for all the Bezier pieces, the positive definite matrix $\mathbf{Q}$ is calculated based on \cite{mellinger2011minimum}, $d^{(l)}_{1,\mu}$ and $d^{(l)}_{N,\mu}$ are the given initial and end states where $l$ denotes the $l$-order derivative, $a_{\min}^g$ and $a_{\max}^g$ are the lower and upper bounds for the $g^{\rm th}$ order derivative of the corresponding piece trajectory, respectively, $j\in\{0,\cdots,n\}$, $i \in \{1, \cdots,N\}$, $\mu \in \{ 1,\cdots,m \}$, $l \in \{ 0,\cdots,n \}$, $\phi \in \{ 0,\cdots,k-1 \}$, $g \in \{ 1,\cdots,k-1 \}$.

In Problem \eqref{nominal_QP}, the objective function is convex and quadratic, and the constraints are all linear. Therefore, it is indeed a convex quadratic program and can be solved efficiently for robotic applications by many off-the-shelf solvers.

\section{Main Method}
 
In this section, we introduce DRSCC-TO in detail. In Section \ref{sec:corridor with dro}, we establish a formal framework for defining a safe corridor in the presence of distributional ambiguity.. In Section \ref{sec:drscc to}, we incorporate the uncertain safe corridor into trajectory optimization and prove that the corresponding optimization problem is still a convex quadratic program.

\subsection{Uncertain Safe Corridor with Distributional Ambiguity}\label{sec:corridor with dro}
As discussed in Remark \ref{SCC_remark}, the SCCs \eqref{safe_constraint} indeed require the true values of $\mathbf{s}^l_i$ and $\mathbf{s}^u_i$. However,  it is difficult (if not impossible) to obtain true values due to the errors of robotic sensors or perception algorithms. Therefore, to model the uncertainty of the safe corridor and distinguish from the deterministic case, we take them as random vectors and denote as $\tilde{\mathbf{s}}^l_i$ and $\tilde{\mathbf{s}}^u_i$, with their distributions as $\mathbb{P}^l_i$ and $\mathbb{P}^u_i$, respectively.
However, it is still challenging to measure the exact distributions   $\mathbb{P}^l_i$ and $\mathbb{P}^u_i$. A popular assumption in probabilistic robotics \cite{thrun2002probabilistic} is to take the distribution of the sensor noise as a Gaussian distribution, but this assumption is too restrictive and possibly inaccurate. 

Instead of a specific distribution, in this paper, we assume that $\mathbb{P}^l_i$ and $\mathbb{P}^u_i$ are \emph{unknown} exactly. In particular, we assume that $\mathbb{P}^l_i$ and $\mathbb{P}^u_i$ reside in the \emph{ambiguity sets} $\mathcal{F}(\hat{\mathbb{P}}^l_i,\theta^l_i)$ and $\mathcal{F}(\hat{\mathbb{P}}^u_i,\theta^u_i)$, respectively, both of which are the Wasserstein balls based on the elliptical reference distributions.
\begin{definition} (Wasserstein distance with  Mahalanobis norm) \\
	The Wasserstein distance between two arbitrary distributions, $\mathbb{P}_1$ and $\mathbb{P}_2$, is defined as
$$
d_{\rm{W}}(\mathbb{P}_1,{\mathbb{P}}_2) = \inf_{\mathbb{P} \in \mathcal{J}(\mathbb{P}_1,\mathbb{P}_2)} \mathbb{E}_{\mathbb{P}}(\Vert \tilde{\mathbf{s}}_1-\tilde{\mathbf{s}}_2 \Vert_{\mathbf{\Sigma}^h_i}).
$$
Here, $\mathcal{J}(\mathbb{P}_1,\mathbb{P}_2)$ is the set of all joint distributions with marginal distributions $\mathbb{P}_1, \mathbb{P}_2$ that govern $\tilde{\mathbf{s}}_1$ and $\tilde{\mathbf{s}}_2$, respectively. The Mahalanobis norm associated with $\mathbf{\Sigma} \succ \mathbf{0}$ 
is $\Vert\mathbf{x}\Vert_{\mathbf{\Sigma}}=\sqrt{\mathbf{x}^\top\mathbf{\Sigma}^{-1}\mathbf{x}}$.
\end{definition}
\begin{assumption}
\label{distributional_ambiguity}
The distributions of $\mathbb{P}^l_i$ and $\mathbb{P}^u_i$ reside in Wasserstein balls with elliptical reference distributions, i.e.,
\begin{equation}\label{prob:wasserstein set}
\begin{array}{l}
\mathbb{P}^h_i\in\mathcal{F}(\hat{\mathbb{P}}^h_i,\theta^h_i) = \{ \mathbb{P}^h_i \in \mathcal{P}(\mathbb{R}^m) \;\vert\; d_{\rm{W}}(\hat{\mathbb{P}}^h_i, \mathbb{P}^h_i)\leq \theta^h_i \}\\
\hfill \forall h \in \{l,u\},
\end{array}
\end{equation}
where $\mathcal{P}(\mathbb{R}^m)$ is the set of all probability distributions on $\mathbb{R}^m$, the reference distribution $\hat{\mathbb{P}}^h_i = \mathbb{P}_{(\pmb{\mu}^h_i,\pmb{\Sigma}^h_i,g^h_i)}$  is a pre-known elliptical reference distribution with a mean vector $\pmb{\mu}^h_i$, a positive definite matrix $\pmb{\Sigma}^h_i$ and a generating function $g(\cdot)$, and $\theta^h_i$ is the pre-known radius for $\mathcal{F}(\hat{\mathbb{P}}^h_i,\theta^h_i)$. 
\end{assumption} 

We emphasize that Assumption \ref{distributional_ambiguity} is mild because the distributions $\mathbb{P}^l_i$ and $\mathbb{P}^u_i$ under Assumption \ref{distributional_ambiguity} are distributionally ambiguous, which are less conservative than assuming a known distribution. Note that in particular, knowing the true distributions $\mathbb{P}^l_i$ and $\mathbb{P}^u_i$ is a special case of Assumption \ref{distributional_ambiguity} (by considering the Wasserstein ball as a singleton that only contains the true distribution, i.e., the reference elliptical distribution). Although the elliptical reference distribution $\hat{\mathbb{P}}^h_i$ is still required to be specified as a prior,  the member distributions in the Wasserstein ball can be any type of distribution  \cite{villani2009optimal}.  

\subsection{DRSCC-TO}\label{sec:drscc to}
Now we are ready to incorporate the distributional ambiguity of the safe corridor into the TO framework.
We leverage the idea of distributionally robust chance constraint, that is, even in the worst-case distribution, the $i$-th piece of the trajectory $B_i(t)$, which is a Bezier curve, is still in the safe corridor with high confidence $1-\epsilon$. Mathematically, we define the Distributioanllay Robust Safe Corridor Constraints (DRSCCs) for $B_i(t)$ as follows:

\begin{definition}
(DRSCC) Suppose Assumption \ref{distributional_ambiguity} holds. We define the distributionally robust counterpart of the nominal SCC \eqref{safe_constraint} as DRSCC:
\begin{subequations}
\label{drscc_original}
\begin{align}
\displaystyle \mathbb{P}_i^l(\mathbf{e}_\mu^\top \tilde{\mathbf{s}}^l_i \leq c^j_{i,\mu}) \geq 1-\epsilon^l_i, & \quad \forall \mathbb{P}_i^l \in \mathcal{F}(\hat{\mathbb{P}}^l_i,\theta^l_i), \label{drcc_original_l_low}\\
\displaystyle \mathbb{P}_i^u(\mathbf{e}_\mu^\top \tilde{\mathbf{s}}^u_i \geq c^j_{i,\mu}) \geq 1-\epsilon^u_i, & \quad \forall \mathbb{P}_i^u \in \mathcal{F}(\hat{\mathbb{P}}^u_i,\theta^u_i), \label{drcc_original_l_up}
\end{align}
\end{subequations}
where $\forall \mu \in \{1,\cdots,m\}$,$ \quad \forall j \in \{0,\cdots,n\}$, and the risk thresholds $\epsilon^l_i, \epsilon^u_i \in(0,0.5)$.
 \end{definition}

Note that our distributionally robust chance constraints do not assume the type of the distributions are specified. Moreover, we leverage the distributionally robust optimality to define DRSCC, that is, the chance constraints hold for all the distributions in the ambiguity sets. It is obvious that, if the ambiguity set comprises a single distribution, the DRSCC is simplified to a chance constraint. Therefore, chance-constrained methods can be seen as a special case of our method.
By replacing the original SCC constraints \eqref{SCC in TO} in \eqref{nominal_QP} with DRSCC \eqref{drscc_original}, we obtain the optimization problem for DRSCC-TO:
\begin{subequations}
\label{robust_QP_original}
\begin{align}
\displaystyle \min_{\mathbf{c}} & \quad \mathbf{c}^\top \mathbf{Q} \mathbf{c},  \\
{\rm{s.t.}} 
&\quad  \eqref{drscc_original}, \\
& \quad \eqref{new_a_1}-\eqref{limit_constraints}.
\end{align}
\end{subequations}
Clearly, Problem \eqref{robust_QP_original} cannot be solved directly because it involves infinitely many constraints in \eqref{drscc_original}. Interestingly, we can prove that it admits a tractable reformulation. Before that, we first define a lower risk threshold for DRSCC:
\begin{equation}
\label{underline_epsilon}
	\underline{\epsilon}^h_i=1-\Phi^h_i(\eta^*)
\end{equation} 
where $h \in \{l,u\}$, $\Phi^h_i$ is the cumulative distribution function for the reference elliptical distribution $\hat{\mathbb{P}}^h_i$, and
\begin{equation}
\label{eta}
\begin{array}{rl}
\eta^* \triangleq 	\min &\eta, \\
{\rm{s.t.}}&\eta \geq (\Phi^h_i)^{-1}(1-\epsilon^h_i), \\
&\eta(\Phi^h_i(\eta)-(1-\epsilon^h_i))- \kappa_i^h \geq \theta^h_i, \\
&\kappa_i^h =\int_{\frac{1}{2}((\Phi^h_i)^{-1}(1-\epsilon^h_i))^2}^{\eta^2/2}k^h_ig^h_i(z){\rm{d}}z.\end{array}
\end{equation}
where $k^h_i$ and $g^h_i(\cdot)$ are the normalizing constant and the generation function for the distribution $\hat{\mathbb{P}}^h_i$, respectively.
Then we have the main theoretical result of this paper:
\begin{theorem}
Under Assumption~\ref{distributional_ambiguity}, the DRSCC-TO optimization problem \eqref{robust_QP_original} is equivalent to
\begin{subequations}
\label{robust_QP_final}
\begin{align}
\displaystyle \min
_{\mathbf{c}} 
& \quad \mathbf{c}^\top \mathbf{Q} \mathbf{c},  \\
{\rm{s.t.}} 
&\quad  c^j_{i,\mu} \geq \mathbf{e}_\mu^\top\pmb{\mu}^l_i + \sqrt{\mathbf{e}^\top_\mu \mathbf{\Sigma}_i^l \mathbf{e}_\mu} (\Phi^l_i)^{-1}(1-\underline{\epsilon}^l_i), \\
&\quad  c^j_{i,\mu} \leq \mathbf{e}_\mu^\top\pmb{{\mu}}^u_i - \sqrt{\mathbf{e}^\top_\mu \mathbf{\Sigma}_i^u \mathbf{e}_\mu} (\Phi^u_i)^{-1}(1-\underline{\epsilon}^u_i), \label{final_lower_bound}\\
& \quad \eqref{waypoint_cons_0}-\eqref{limit_constraints}.
\end{align}
\end{subequations}
\end{theorem}
\begin{proof}
	Please see the Appendix.
\end{proof}
Note that the problem \eqref{robust_QP_final} is still a convex QP. Therefore, similar to the nominal TO \eqref{nominal_QP}, we could also solve our DRSCC-TO \eqref{robust_QP_final} in real-time for robotic applications. 

\section{Results and Discussion}
In this section, we conduct both simulation and real robot tests to evaluate our approach. Please see our complementary video for more implementation details. \footnote{See also \url{www.xushaohang.top/home/drscc-to}.}
\subsection{Numerical Simulation}
\begin{table*}[t]
\caption{Numerical Simulation Results}
\centering
\begin{tabular}{c|c|c|c|c|c|c|c}
\hline
 \multicolumn{1}{c|}{}& \textbf{Nominal TO} &\multicolumn{6}{|c}{\textbf{DRSCC-TO}} \\
\hline
\hline
\multirow{2}{*}{\textbf{Normal Distribution}} & \multirow{2}{*}{\diagbox[width=8em]{\ }{\ }} &  \multicolumn{3}{c|}{${\theta=0.05}$} & \multicolumn{3}{c}{${\theta=0.1}$} \\
\cline{3-8}
& & ${\epsilon=0.1}$& ${\epsilon=0.15}$  & ${\epsilon=0.25}$ & ${\epsilon=0.1}$& ${\epsilon=0.15}$  & ${\epsilon=0.25}$ \\
\cline{1-3}\cline{4-8}
\textbf{Optimal Objective}& 1 & 2.196 & 1.871 &  1.619 & 2.544 & 2.009 & 1.732 \\
\cline{1-8}
\textbf{Number of Violation}  & 2478 & 129 & 215 & 410 & 96 & 151& 292\\
\hline
\hline
\multirow{2}{*}{\textbf{t-Distribution}} & \multirow{2}{*}{\diagbox[width=8em]{\ }{\ }} &  \multicolumn{3}{c| }{${\theta=0.05}$} & \multicolumn{3}{c}{${\theta=0.1}$} \\
\cline{3-8}
& & ${\epsilon=0.1}$& ${\epsilon=0.15}$  & ${\epsilon=0.25}$ & ${\epsilon=0.1}$& ${\epsilon=0.15}$  & ${\epsilon=0.25}$ \\
\cline{1-3}\cline{4-8}
\textbf{Optimal Objective}& 1 & 9.022 & 4.378 &  2.217 & 13.515 & 5.560 & 2.649 \\
\cline{1-8}
\textbf{Number of Violation}  & 5739 & 767 & 798 & 1456 & 755 & 775& 1260\\
\hline
\hline
\multirow{2}{*}{\textbf{Logistic Distribution}} & \multirow{2}{*}{\diagbox[width=8em]{\ }{\ }} &  \multicolumn{3}{c| }{${\theta=0.05}$} & \multicolumn{3}{c}{${\theta=0.1}$} \\
\cline{3-8}
& & ${\epsilon=0.1}$& ${\epsilon=0.15}$  & ${\epsilon=0.25}$ & ${\epsilon=0.1}$& ${\epsilon=0.15}$  & ${\epsilon=0.25}$ \\
\cline{1-3}\cline{4-8}
\textbf{Optimal Objective}& 1 & 2.266 & 1.934 &  1.568 & 2.553 & 2.117 &1.697 \\
\cline{1-8}
\textbf{Number of Violation}  & 5718 & 5 & 60 & 1302 & 5 & 9& 482\\
\hline

\end{tabular}
\end{table*}

In numerical simulation, we compare our DRSCC-TO method with the nominal TO method in terms of optimal costs and violation count. We focus on 2D minimum snap trajectory optimization \cite{mellinger2011minimum} and proceed as follows.

First, we establish three nominal safe corridors, which are known for both nominal TO and DRSCC-TO. For DRSCC-TO, we employ different reference distributions: normal distribution, t-distribution, and logistic distribution. For these ellipsoidal distributions, the mean vectors align with corner positions of nominal safe corridors, and the covariance matrices are $\sigma I$, with $\sigma$ values being $2$, $1$, and $1$ for each distribution. Then, we formulate the problems \eqref{nominal_QP} and \eqref{robust_QP_final} for nominal TO and DRSCC-TO. Solving them produces optimized trajectories. DRSCC-TO employs consistent Wasserstein distance and confidence level for each sub-region and dimension. We compare six parameter sets for DRSCC-TO ($d \in \{ 0.05, 0.1\}$, $\epsilon \in \{ 0.1, 0.15, 0.25 \}$). Finally, nominal safe corridors are perturbed to assess method robustness.  We sample $\psi_1$ from the prescribed reference distribution and $\psi_2$ from a uniform distribution. The instance $\psi = (1-\alpha) \psi_1 + \alpha \psi_2$ is obtained using weight $\alpha \in [0,1]$, with five $\alpha$ values and 2000 instances per value, resulting in 10000 \emph{perturbed} safe corridors in each test.

We calculate optimal costs for both methods across different demonstrations and the results are summarized in Table I. DRSCC-TO exhibits significantly fewer violations than nominal TO. Especially in \emph{Case III}, DRSCC-TO with logistic distribution ($\theta=0.5$, $\epsilon=0.1$) rarely shows infeasible motions compared to nominal TO. Also, optimal cost decreases with smaller $d$ or larger $\epsilon$, but violations increase. In conclusion, our method improves safety by reducing violations when safe corridors are perturbed.

The trade-off is that our method's optimal cost exceeds nominal TO's. This implies our method's control inputs are more energetic. However, safety generally takes precedence over efficiency for autonomous robots, as infeasible motions risk hardware damage (e.g., micro drones colliding with obstacles).  

\subsection{Robotic Applications}
In this section, we further validate our method on a micro UAV and a quadruped robot to show the practical viability of our method. 

\subsubsection{UAV}
The UAV in this test is shown in Fig. \ref{fig:robot}. The high-level trajectory optimizer and tracking controller are implemented on a mini PC with Intel i7 CPU, and the low-level controller runs on a PX4 flight controller. Moreover, we leverage the inertial measurement unit in the PX4 controller, the depth camera RealSense D435i and the VINS algorithm \cite{qin2018vins} to locate the robot in the 3D space.
\begin{figure}
\centering
\includegraphics[width=8.5 cm]{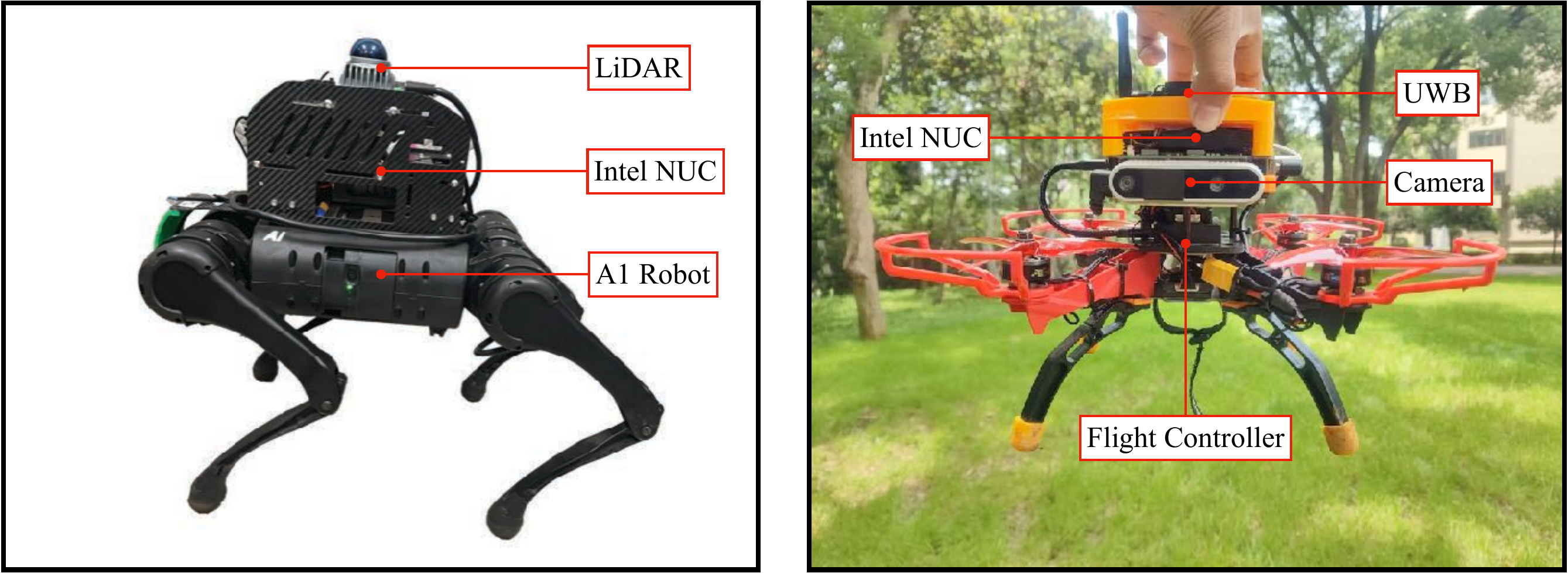}
\caption{The robots used in our experiments: the quadruped robot Unitree A1 (left) and a micro UAV (right). }
\label{fig:robot}
\end{figure}
In this test, we manually put some static obstacles in the scene, and input the perfect perception information (i.e., the safe corridor) into the TO method to obtain the optimized trajectory. 
We utilize Gaussian distribution as the reference distribution for all the sub-feasible regions. The mean vector of the Gaussian distributions is coincident
 with the perfect safe corridor, and the variance  0.4. Moreover, we set $\epsilon = 0.25$ and $\theta = 0.1$ to balance the safety and the efficiency.
\begin{figure}
\centering
\subfigure[]{
\includegraphics[angle=90,width=6cm]{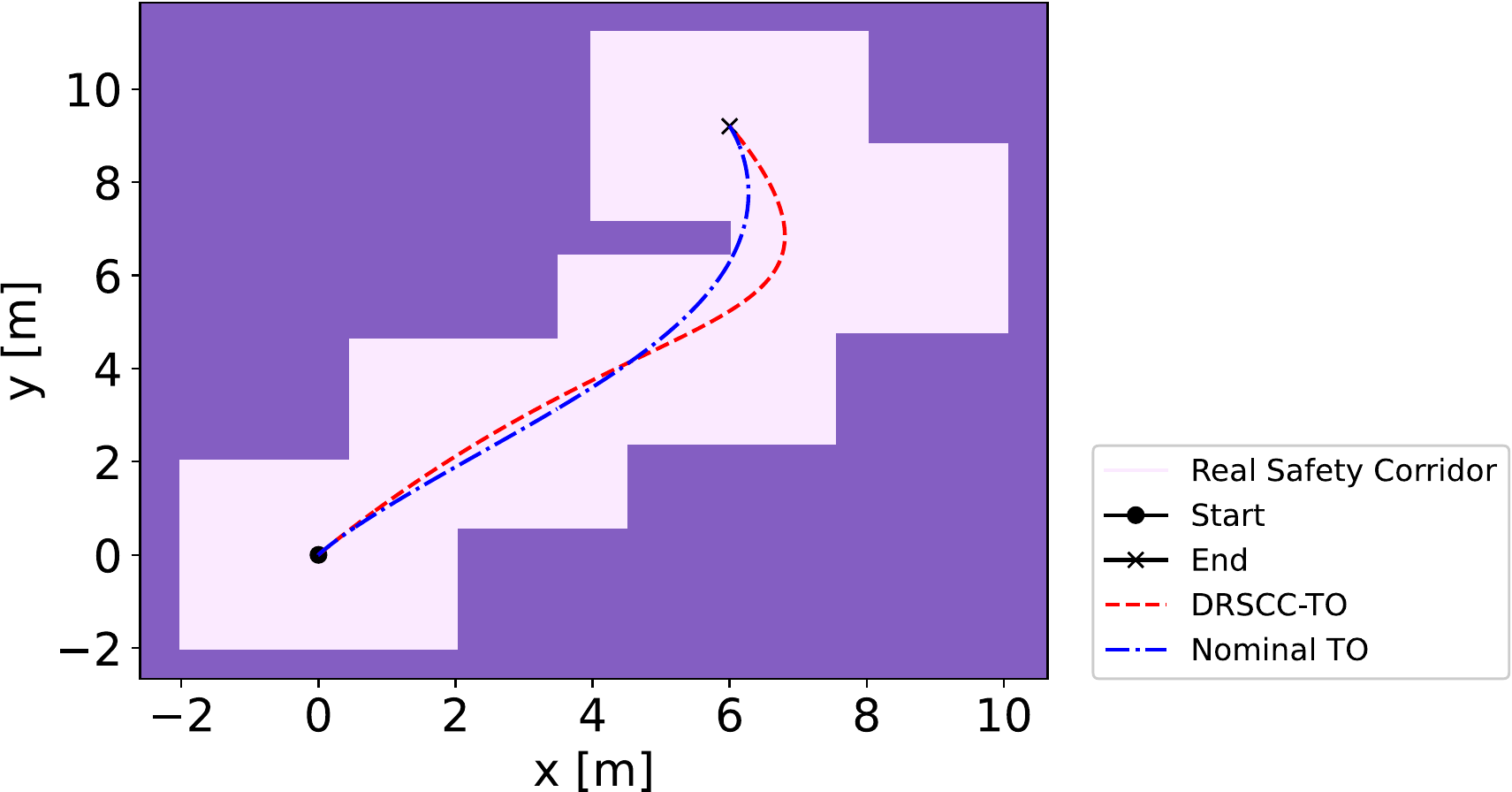}
\label{fig:uav_compare}
}
\subfigure[]{
\includegraphics[width=4cm]{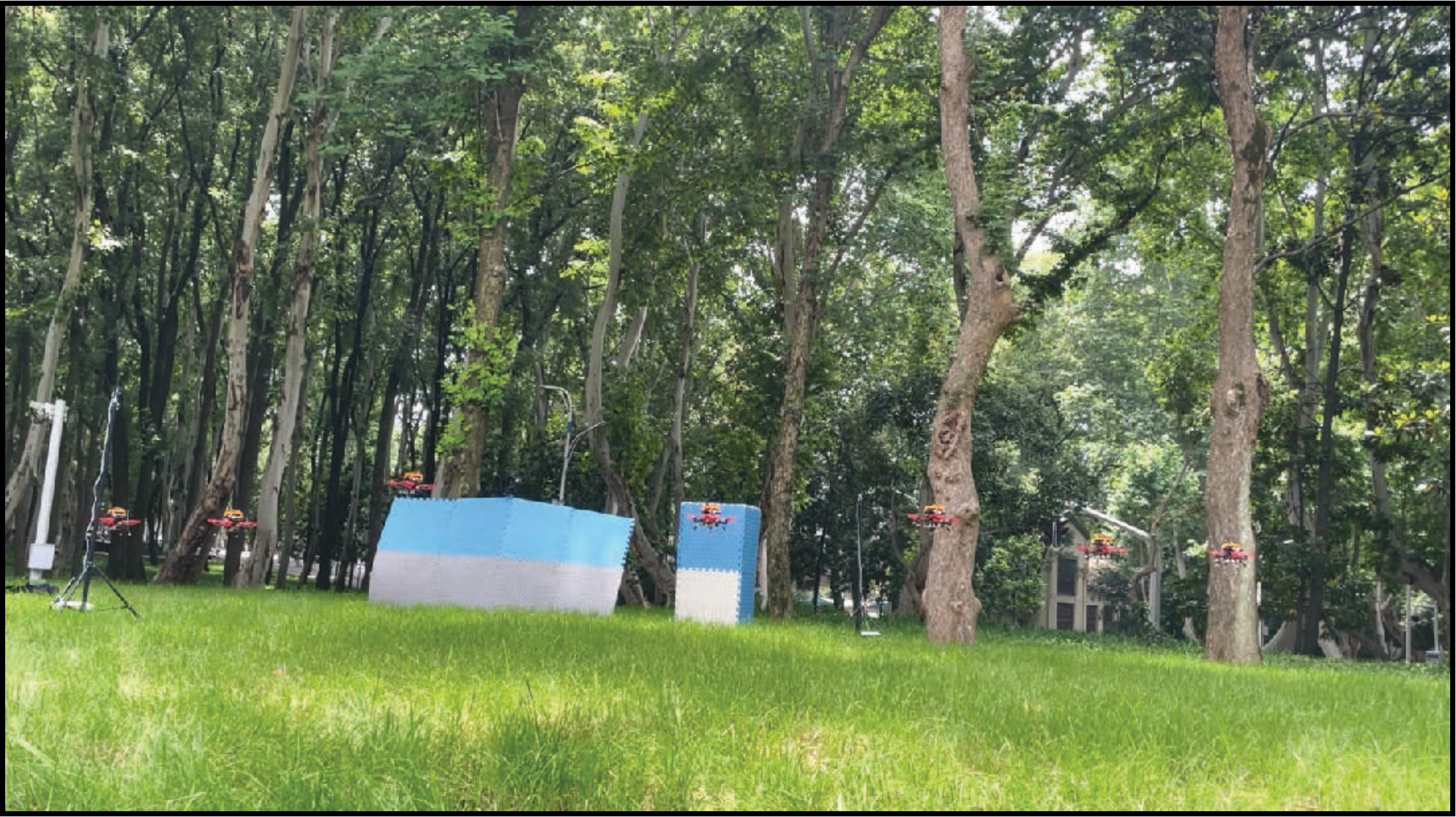}
\label{fig:uav_nominal}
}
\subfigure[]{
\includegraphics[width=4cm]{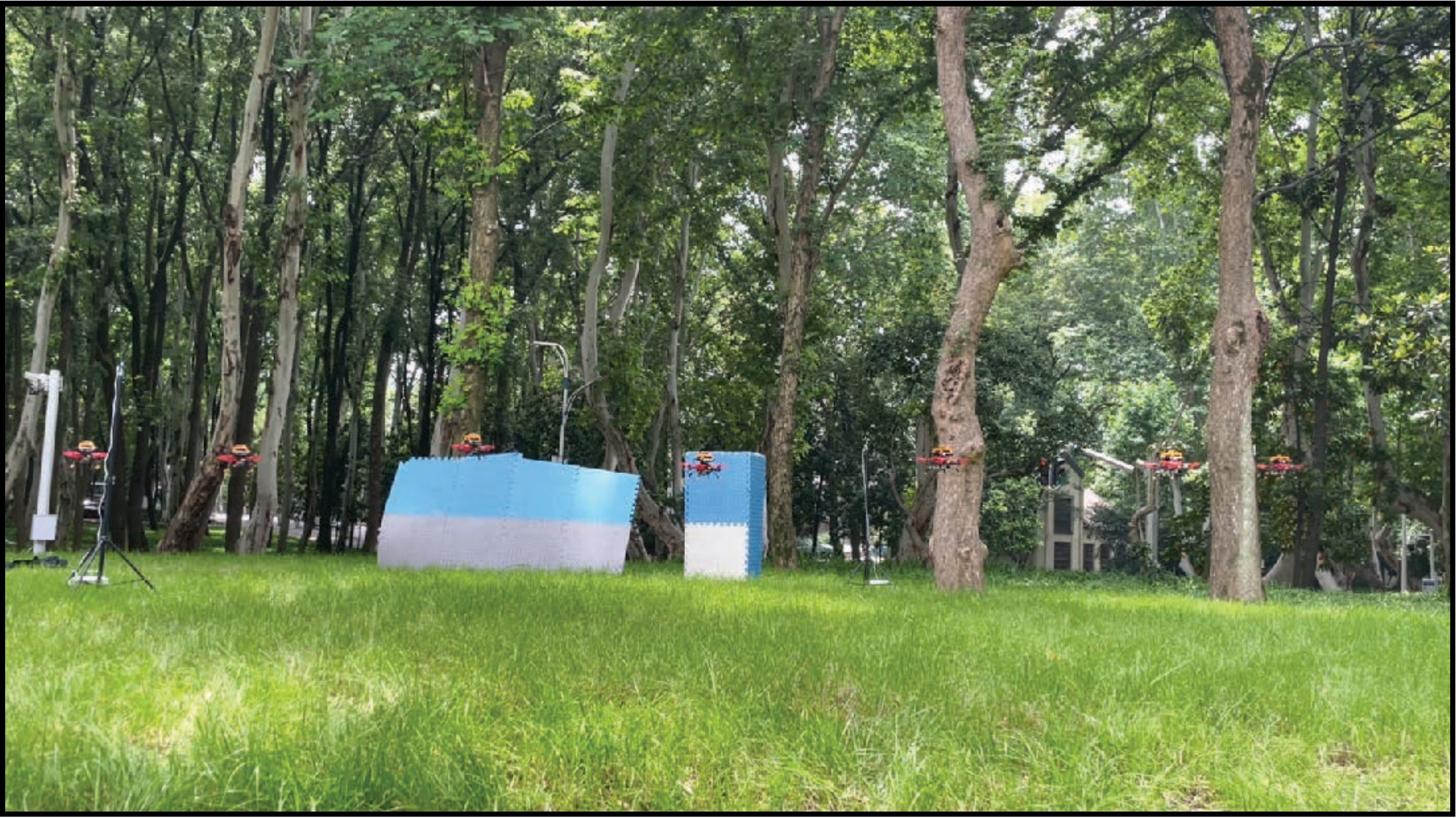}\label{fig:uav_robust}
}
\caption{Experimental Results on our UAV. (a) Comparison of the optimized trajectories. (b) Nominal TO with the exact safe corridor. (c) DRSCC-TO with the exact safe corridor. Both methods are collision-free with the exact safety corridor, although the optimized trajectory of the nominal TO is closer to the boundary.}
\label{fig:uav_experiment}
\end{figure}

The results are shown in Fig. \ref{fig:uav_experiment}.    In Fig. \ref{fig:uav_compare}, two optimized trajectories in the x-y plane are plotted, which are obtained from the nominal TO and DRSCC-TO, respectively. It is clear that the trajectory from the nominal TO is closer to the boundary of the safe corridor than that from DRSCC-TO. The experimental results are shown in Fig. \ref{fig:uav_nominal} and \ref{fig:uav_robust}. It could be seen that when the safe corridor is deterministic without any disturbance, both methods are collision-free with the same flight controller and hardware. Intuitively, the nominal TO will have better efficiency because it generates a shorter trajectory compared to DRSCC-TO. 
The nominal TO is more sensitive to the safe corridor information than DRSCC-TO because the trajectory could immediately become infeasible if the safe corridor is disturbed by the perception errors.

\subsubsection{Quadruped Robot}

\begin{figure}
\centering
\subfigure[]{
\includegraphics[width=8cm]{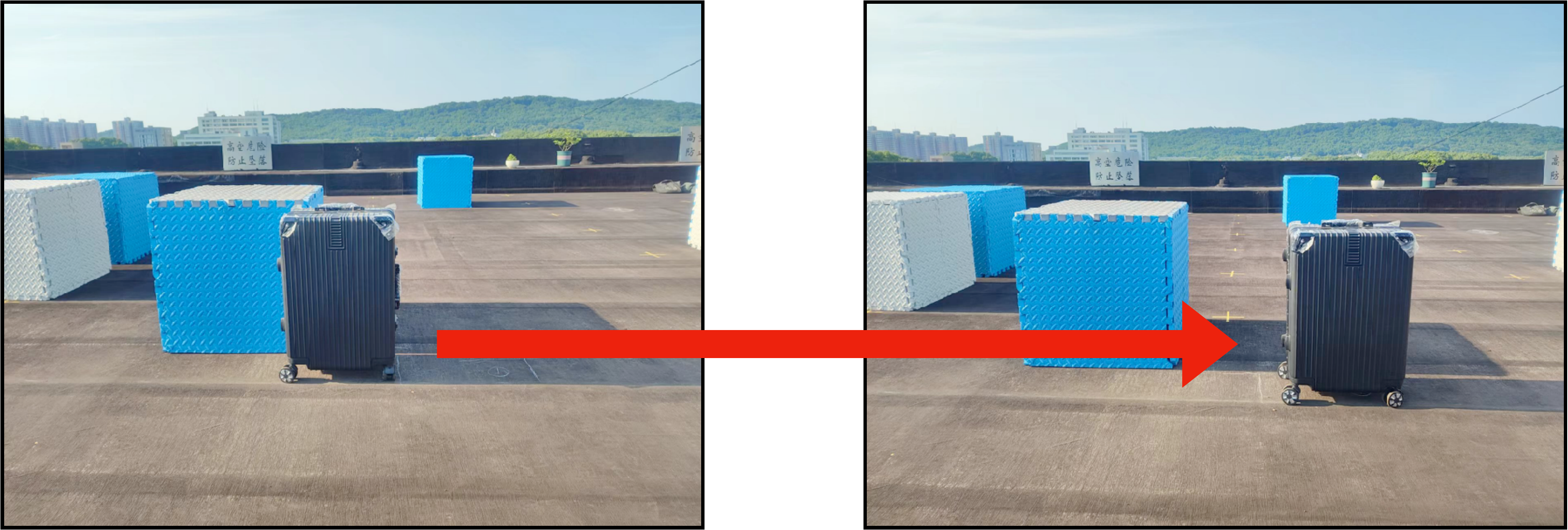}
\label{fig:moving_box}
}
\subfigure[]{
\includegraphics[width=4cm]{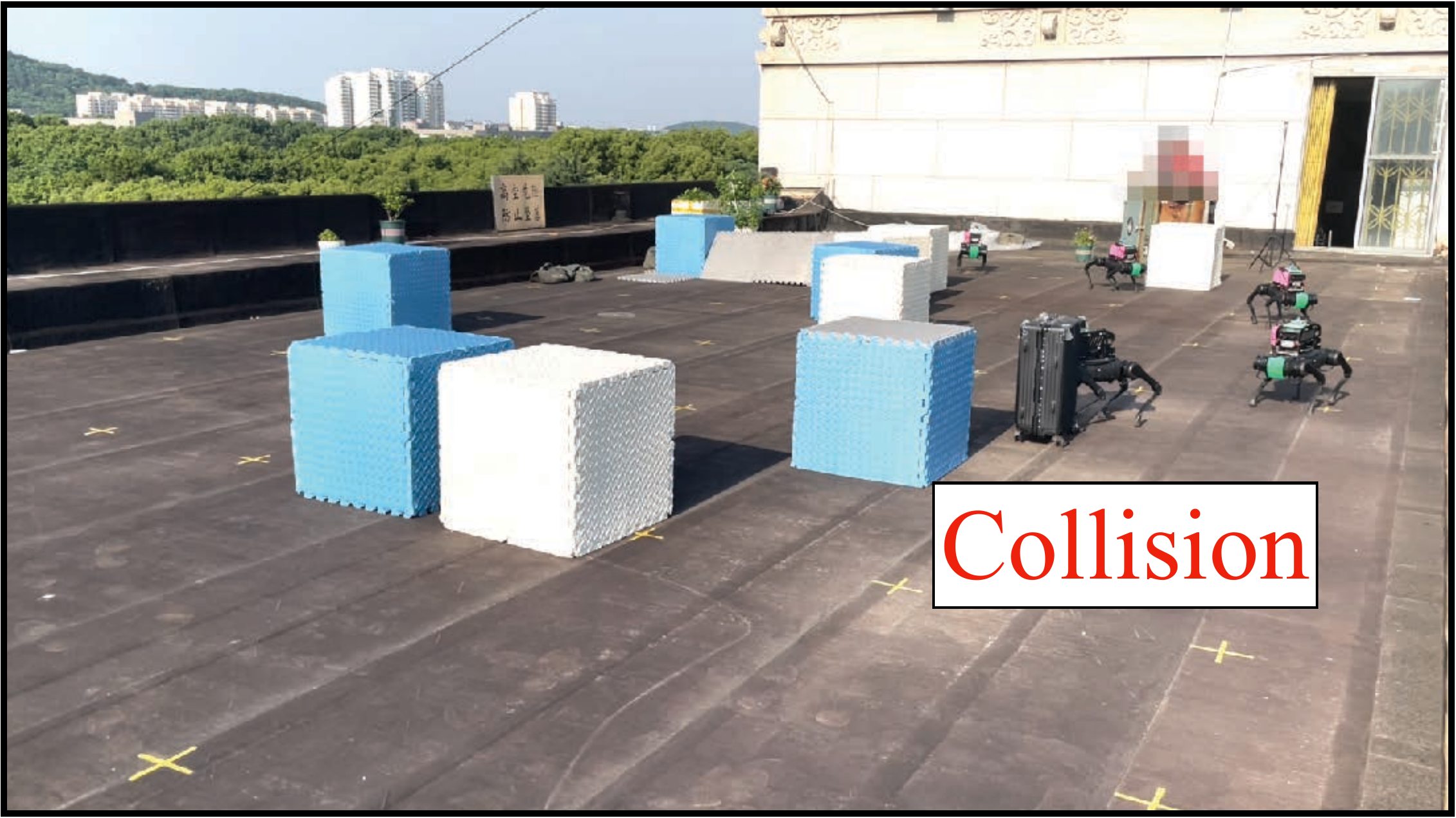}\label{fig:nominal_disturb}
}
\subfigure[]{
\includegraphics[width=4cm]{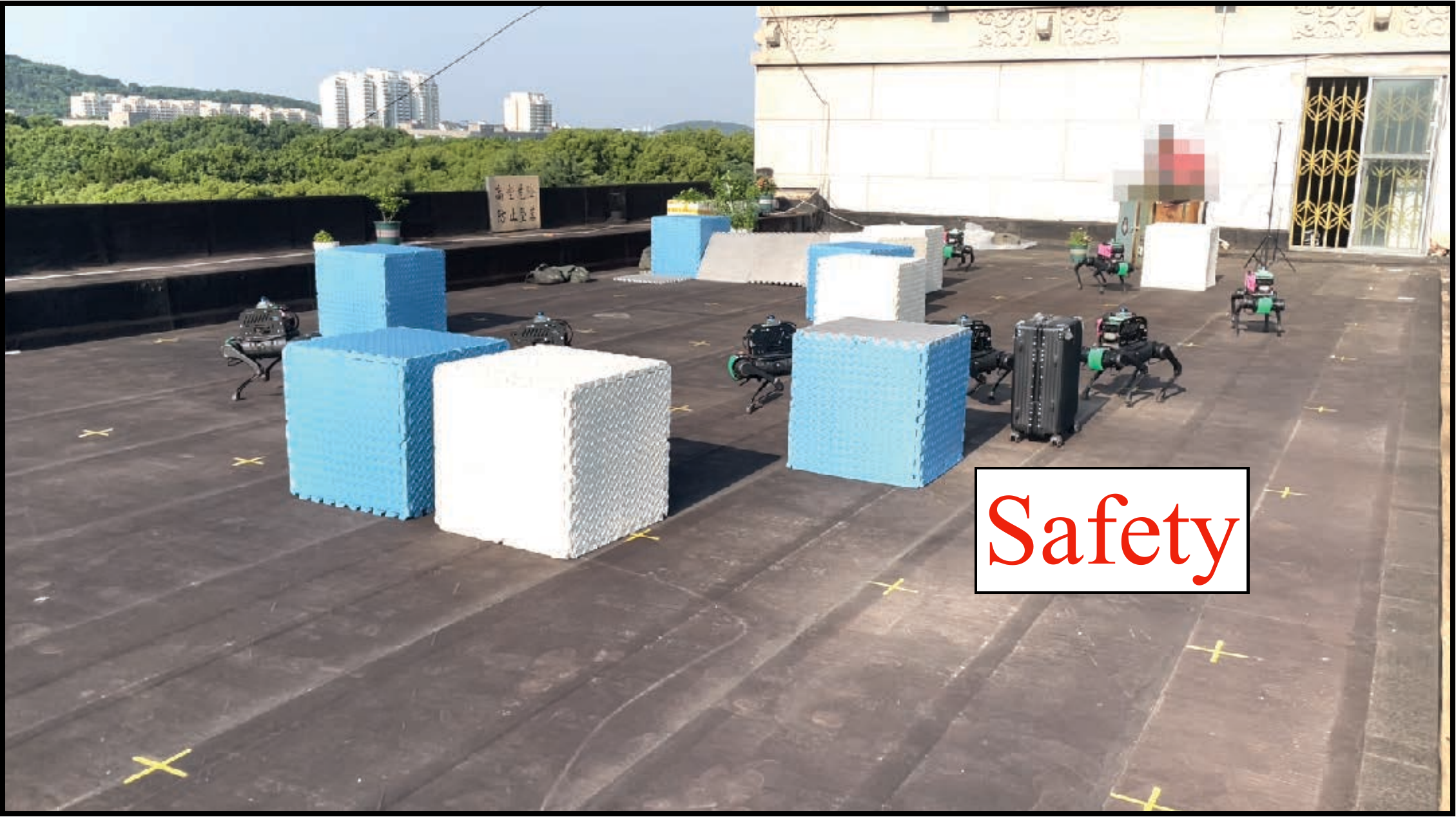}
\label{fig:robust_disturb}
}
\caption{Experimental Results on Unitree A1. 
(a) The safe corridor is disturbed by moving the black obstacle.
(b) Nominal TO with the disturbed safe corridor. (c) DRSCC-TO with the disturbed safe corridor. }
\label{fig:quad}
\end{figure}
We give another application example of quadruped navigation, in order to show how the nominal TO will lead to cataphatic failures when the safe corridor is inaccurate but our method will still succeed. 

The quadruped robot, Unitree A1, is shown in Fig. \ref{fig:robot}. We leverage a LiDAR and the Fast-LIO2 algorithm \cite{xu2022fast} to locate the robot in 2D space. We utilize a nonlinear model predictive controller based on the OCS2 toolbox \cite{OCS2} for quadruped locomotion control. All the algorithms are implemented on an onboard computer with the Intel i7 CPU. All the hyper-parameters of our DRSCC-TO are the same as those of the UAV test. First, we compare the nominal TO with DRSCC-TO similar to the UAV test: we assume that there are no perception errors and that the perfect safe corridor is used for both methods. The results are similar to the UAV test:  both methods are collision-free while the robot with the nominal TO is very close to some obstacles. It is not surprising since the obstacle map is perfectly known.

Then, we disturb the safe corridor by moving one obstacle slightly, as shown in Fig. \ref{fig:moving_box}. 
We aim to show how both methods perform when the exact safe corridor is disturbed. The results of the second test are shown in Fig. \ref{fig:nominal_disturb} and \ref{fig:robust_disturb}. With DRSCC-TO, the robot could also navigate safely without any collision. However, the robot with the nominal method collides with the black suitcase severely. 
Therefore, we can conclude from this test that our method is practically useful for safety-critical robot navigation to reduce the risk of collisions induced by perception errors.
\section{Conclusion}
In this paper, we propose a robust trajectory optimization framework, DRSCC-TO, for mobile robots navigating in uncertain environments. In particular, we propose distributionally robust chance constraints to deal with the uncertainty of the safe corridor and then incorporate the constraints into the trajectory optimization problem which is proved to be a convex quadratic program. The results of numerical simulation and robot experiments show that our method significantly reduces infeasible motions and improves the safety of robot autonomy. 

\section*{Appendix}
\subsection*{Proof of Theorem 1:}
Notice that \eqref{drcc_original_l_low} is equivalent to 
\begin{equation*}
\sup_{\mathbb{P}_i^l \in \mathcal{F}(\hat{\mathbb{P}}^l_i,\theta^l_i)} \mathbb{P}_i^l(\mathbf{e}_\mu^\top \tilde{\mathbf{s}}^l_i > c^j_{i,\mu}) \leq \epsilon,
\end{equation*}
where the strict inequality can be replaced by a weak one and we have 
\begin{equation}\label{prob:drcc}
\sup_{\mathbb{P}_i^l \in \mathcal{F}(\hat{\mathbb{P}}^l_i,\theta^l_i)} \mathbb{P}_i^l(\mathbf{e}_\mu^\top \tilde{\mathbf{s}}^l_i \geq c^j_{i,\mu}) \leq \epsilon,
\end{equation}
see, e.g., proposition 3 in \cite{gao2022distributionally}. To proceed, we kindly remind the definition of the value-at-risk:
$$
\mathbb{P}\text{-VaR}_{1-\epsilon}(\tilde{s})=\inf_{x\in\mathbb{R}}\{x\;\vert\;\mathbb{P}(\tilde{s}>x)\leq\epsilon\}
$$
with risk threshold $\epsilon\in(0,0.5)$. Then, we can rewrite~\eqref{prob:drcc} as:
\begin{equation}\label{prob:drcc lhs}
\sup_{\mathbb{P}_i^l \in \mathcal{F}(\hat{\mathbb{P}}^l_i,\theta^l_i)} \mathbb{P}_i^l(\mathbf{e}_\mu^\top \tilde{\mathbf{s}}^l_i - c^j_{i,\mu} \geq 0) \leq \epsilon,
\end{equation}
which is equivalent to
\begin{equation}\label{prob:worst case var}
\sup_{\mathbb{P}_i^l \in \mathcal{F}(\hat{\mathbb{P}}^l_i,\theta^l_i)} \mathbb{P}^l_i \text{-VaR}_{1-\epsilon}(\mathbf{e}_\mu^\top\tilde{\mathbf{s}}^l_i-c^j_{i,\mu})\leq 0.
\end{equation}
Note that the equivalence follows immediately from the definition of the value-at-risk.
Due to the choice of the Mahalanobis norm, we can establish the equivalence between \eqref{prob:worst case var} and 
\begin{equation}
\label{prob:cc nominal}
\mathbb{P}_{\pmb{\mu}^l_i,\pmb{\Sigma}^l_i,g^l_i}(\mathbf{e}_\mu^\top\tilde{\mathbf{s}}^l_i\leq c^j_{i,\mu})\geq 1-\underline{\epsilon}^l_i.
\end{equation}
Please see corollary 4.9 in \cite{chen2021sharing} and also \cite{ruan2023risk} for the details of the proof of the equivalence.  Remember in Assumption~\ref{distributional_ambiguity} that the reference distribution is assumed to be an elliptical distribution $\hat{\mathbb{P}}^l_i=\mathbb{P}_{\pmb{\mu}^l_i,\pmb{\Sigma}^l_i,g^l_i}$. Also, from \eqref{underline_epsilon} and \eqref{eta}, we have:
\begin{equation}
\label{underline_epsilon_less_0.5}
	\begin{array}{rl}
		\underline{\epsilon}^l_i &= 1-\Phi^h_i(\eta^*) \\
		&\leq 1-\Phi^h_i((\Phi^h_i)^{-1}(1-\epsilon^l_i)) \\
	&=\epsilon^l_i<0.5.
	\end{array}
\end{equation}
Therefore, we have the equivalences as follows:
\begin{equation*}
\begin{array}{rl}
&\displaystyle \hat{\mathbb{P}}^l_i(\mathbf{e}_\mu^\top \tilde{\mathbf{s}}^l_i \leq c^j_{i,\mu}) \geq 1-\underline{\epsilon}^l_i  \vspace{1mm}\\
\Longleftrightarrow &\mathbb{P}_{\pmb{\mu}^l_i,\pmb{\Sigma}^l_i,g^l_i}(\mathbf{e}_\mu^\top\tilde{\mathbf{s}}^l_i\leq c^j_{i,\mu})\geq 1-\underline{\epsilon}^l_i \vspace{1mm} \\
\Longleftrightarrow& \Phi\left(\frac{c^j_{i,\mu}-\mathbf{e}_\mu^\top\bm{\mu}^l_i}{\sqrt{\mathbf{e}_\mu^\top\mathbf{\Sigma}^l_i\mathbf{e}_\mu}}\right) \geq 1-\underline{\epsilon}^l_i \vspace{1mm}\\
\Longleftrightarrow& c^j_{i,\mu} \geq \mathbf{e}_\mu^\top\pmb{\mu}^l_i + \sqrt{\mathbf{e}^\top_\mu \mathbf{\Sigma}_i^l \mathbf{e}_\mu} (\Phi^l_i)^{-1}(1-\underline{\epsilon}^l_i), 
\end{array}
\end{equation*}
where the second equivalence holds for the linearity of elliptical distributions, and the third one is due to \eqref{underline_epsilon_less_0.5}.
Similarly, we can prove that \eqref{drcc_original_l_up} is equivalent to:
\begin{equation*}
c^j_{i,\mu} \leq \mathbf{e}_\mu^\top\pmb{{\mu}}^u_i - \sqrt{\mathbf{e}^\top_\mu \mathbf{\Sigma}_i^u \mathbf{e}_\mu} (\Phi^u_i)^{-1}(1-\underline{\epsilon}^u_i), 
\end{equation*}
Therefore, we conclude that the problem \eqref{robust_QP_original} is equivalent to problem \eqref{robust_QP_final}.


\bibliographystyle{IEEEtran}
\bibliography{ref}

\begin{thebibliography}{10}
\providecommand{\url}[1]{#1}
\csname url@rmstyle\endcsname
\providecommand{\newblock}{\relax}
\providecommand{\bibinfo}[2]{#2}
\providecommand\BIBentrySTDinterwordspacing{\spaceskip=0pt\relax}
\providecommand\BIBentryALTinterwordstretchfactor{4}
\providecommand\BIBentryALTinterwordspacing{\spaceskip=\fontdimen2\font plus
\BIBentryALTinterwordstretchfactor\fontdimen3\font minus
  \fontdimen4\font\relax}
\providecommand\BIBforeignlanguage[2]{{%
\expandafter\ifx\csname l@#1\endcsname\relax
\typeout{** WARNING: IEEEtran.bst: No hyphenation pattern has been}%
\typeout{** loaded for the language `#1'. Using the pattern for}%
\typeout{** the default language instead.}%
\else
\language=\csname l@#1\endcsname
\fi
#2}}

\bibitem{gao2018online}
F.~Gao, W.~Wu, Y.~Lin, and S.~Shen, ``Online safe trajectory generation for
  quadrotors using fast marching method and bernstein basis polynomial,'' in
  \emph{2018 IEEE International Conference on Robotics and Automation
  (ICRA)}.\hskip 1em plus 0.5em minus 0.4em\relax IEEE, 2018, pp. 344--351.

\bibitem{deits2015efficient}
R.~Deits and R.~Tedrake, ``Efficient mixed-integer planning for uavs in
  cluttered environments,'' in \emph{2015 IEEE international conference on
  robotics and automation (ICRA)}.\hskip 1em plus 0.5em minus 0.4em\relax IEEE,
  2015, pp. 42--49.

\bibitem{chen2016online}
J.~Chen, T.~Liu, and S.~Shen, ``Online generation of collision-free
  trajectories for quadrotor flight in unknown cluttered environments,'' in
  \emph{2016 IEEE International Conference on Robotics and Automation
  (ICRA)}.\hskip 1em plus 0.5em minus 0.4em\relax IEEE, 2016, pp. 1476--1483.

\bibitem{liu2017planning}
S.~Liu, M.~Watterson, K.~Mohta, K.~Sun, S.~Bhattacharya, C.~J. Taylor, and
  V.~Kumar, ``Planning dynamically feasible trajectories for quadrotors using
  safe flight corridors in 3-d complex environments,'' \emph{IEEE Robotics and
  Automation Letters}, vol.~2, no.~3, pp. 1688--1695, 2017.

\bibitem{sun2021fast}
W.~Sun, G.~Tang, and K.~Hauser, ``Fast uav trajectory optimization using
  bilevel optimization with analytical gradients,'' \emph{IEEE Transactions on
  Robotics}, vol.~37, no.~6, pp. 2010--2024, 2021.

\bibitem{ben2009robust}
A.~Ben-Tal, L.~El~Ghaoui, and A.~Nemirovski, \emph{Robust optimization}.\hskip
  1em plus 0.5em minus 0.4em\relax Princeton university press, 2009.

\bibitem{hakobyan2021wasserstein}
A.~Hakobyan and I.~Yang, ``Wasserstein distributionally robust motion control
  for collision avoidance using conditional value-at-risk,'' \emph{IEEE
  Transactions on Robotics}, vol.~38, no.~2, pp. 939--957, 2021.

\bibitem{deits2015computing}
R.~Deits and R.~Tedrake, ``Computing large convex regions of obstacle-free
  space through semidefinite programming,'' in \emph{Algorithmic Foundations of
  Robotics XI: Selected Contributions of the Eleventh International Workshop on
  the Algorithmic Foundations of Robotics}.\hskip 1em plus 0.5em minus
  0.4em\relax Springer, 2015, pp. 109--124.

\bibitem{mellinger2011minimum}
D.~Mellinger and V.~Kumar, ``Minimum snap trajectory generation and control for
  quadrotors,'' in \emph{2011 IEEE international conference on robotics and
  automation}.\hskip 1em plus 0.5em minus 0.4em\relax IEEE, 2011, pp.
  2520--2525.

\bibitem{mayne2005robust}
D.~Q. Mayne, M.~M. Seron, and S.~Rakovi{\'c}, ``Robust model predictive control
  of constrained linear systems with bounded disturbances,'' \emph{Automatica},
  vol.~41, no.~2, pp. 219--224, 2005.

\bibitem{majumdar2017funnel}
A.~Majumdar and R.~Tedrake, ``Funnel libraries for real-time robust feedback
  motion planning,'' \emph{The International Journal of Robotics Research},
  vol.~36, no.~8, pp. 947--982, 2017.

\bibitem{charnes1959chance}
A.~Charnes and W.~W. Cooper, ``Chance-constrained programming,''
  \emph{Management science}, vol.~6, no.~1, pp. 73--79, 1959.

\bibitem{blackmore2011chance}
L.~Blackmore, M.~Ono, and B.~C. Williams, ``Chance-constrained optimal path
  planning with obstacles,'' \emph{IEEE Transactions on Robotics}, vol.~27,
  no.~6, pp. 1080--1094, 2011.

\bibitem{manchester2017dirtrel}
Z.~Manchester and S.~Kuindersma, ``Dirtrel: Robust trajectory optimization with
  ellipsoidal disturbances and lqr feedback.'' in \emph{Robotics: Science and
  Systems}.\hskip 1em plus 0.5em minus 0.4em\relax Cambridge, MA, USA, 2017.

\bibitem{nakka2022trajectory}
Y.~K. Nakka and S.-J. Chung, ``Trajectory optimization of chance-constrained
  nonlinear stochastic systems for motion planning under uncertainty,''
  \emph{IEEE Transactions on Robotics}, 2022.

\bibitem{gao2016online}
F.~Gao and S.~Shen, ``Online quadrotor trajectory generation and autonomous
  navigation on point clouds,'' in \emph{2016 IEEE International Symposium on
  Safety, Security, and Rescue Robotics (SSRR)}.\hskip 1em plus 0.5em minus
  0.4em\relax IEEE, 2016, pp. 139--146.

\bibitem{gao2019flying}
F.~Gao, W.~Wu, W.~Gao, and S.~Shen, ``Flying on point clouds: Online trajectory
  generation and autonomous navigation for quadrotors in cluttered
  environments,'' \emph{Journal of Field Robotics}, vol.~36, no.~4, pp.
  710--733, 2019.

\bibitem{thrun2002probabilistic}
S.~Thrun, ``Probabilistic robotics,'' \emph{Communications of the ACM},
  vol.~45, no.~3, pp. 52--57, 2002.

\bibitem{villani2009optimal}
C.~Villani \emph{et~al.}, \emph{Optimal transport: old and new}.\hskip 1em plus
  0.5em minus 0.4em\relax Springer, 2009, vol. 338.

\bibitem{qin2018vins}
T.~Qin, P.~Li, and S.~Shen, ``Vins-mono: A robust and versatile monocular
  visual-inertial state estimator,'' \emph{IEEE Transactions on Robotics},
  vol.~34, no.~4, pp. 1004--1020, 2018.

\bibitem{xu2022fast}
W.~Xu, Y.~Cai, D.~He, J.~Lin, and F.~Zhang, ``Fast-lio2: Fast direct
  lidar-inertial odometry,'' \emph{IEEE Transactions on Robotics}, vol.~38,
  no.~4, pp. 2053--2073, 2022.

\bibitem{OCS2}
F.~Farshidian \emph{et~al.}, ``{OCS2}: An open source library for optimal
  control of switched systems,'' [Online]. Available:
  \url{https://github.com/leggedrobotics/ocs2}.

\bibitem{gao2022distributionally}
R.~Gao and A.~Kleywegt, ``Distributionally robust stochastic optimization with
  wasserstein distance,'' \emph{Mathematics of Operations Research}, 2022.

\bibitem{chen2021sharing}
Z.~Chen and W.~Xie, ``Sharing the value-at-risk under distributional
  ambiguity,'' \emph{Mathematical Finance}, vol.~31, no.~1, pp. 531--559, 2021.

\bibitem{ruan2023risk}
H.~Ruan, Z.~Chen, and C.~P. Ho, ``Risk-averse mdps under reward ambiguity,''
  \emph{arXiv preprint arXiv:2301.01045}, 2023.

\end{thebibliography}
\end{document}